\documentclass{article}
\usepackage{makecell}
\usepackage{microtype}
\usepackage{graphicx}
\usepackage[hang]{subfigure}
\usepackage{booktabs}
\usepackage{tcolorbox}
\usepackage{nicefrac}
\usepackage[export]{adjustbox}
\usepackage{dsfont}
\usepackage{hyperref}
\usepackage{xurl}

\newcommand{\bigO}{\mathcal{O}}

\usepackage[accepted]{icml2025_arxiv}

\usepackage{amsmath}
\usepackage{amssymb}
\usepackage{mathtools}
\usepackage{amsthm}
\usepackage{bm}
\usepackage{algorithm}
\usepackage{algorithmic}
\usepackage{nicefrac}

\usepackage[capitalize,noabbrev]{cleveref}

\theoremstyle{plain}
\newtheorem{theorem}{Theorem}[section]

\theoremstyle{definition}

\newtheorem{assumption}[theorem]{Assumption}
\theoremstyle{remark}

\usepackage[textsize=tiny]{todonotes}

\begin{document}

\twocolumn[
\icmltitle{Mean-Field Bayesian Optimisation}

\begin{icmlauthorlist}
\icmlauthor{Petar Steinberg\textsuperscript{*}}{ucl}
\icmlauthor{Juliusz Ziomek\textsuperscript{*}}{oxford}
\icmlauthor{Matej Jusup}{ethz}
\icmlauthor{Ilija Bogunovic}{ucl}
\end{icmlauthorlist}

\icmlaffiliation{ucl}{University College London}
\icmlaffiliation{oxford}{University of Oxford}
\icmlaffiliation{ethz}{ETH Zurich}

\icmlcorrespondingauthor{Juliusz Ziomek}{juliusz (dot) ziomek [at] univ (dot) ox [dot] ac (dot) uk}

\icmlkeywords{Machine Learning}

\vskip 0.3in
]

\printAffiliationsAndNotice{\icmlEqualContribution}

\begin{abstract}
We address the problem of optimising the average payoff for a large number of cooperating agents, where the payoff function is unknown and treated as a black box. While standard Bayesian Optimisation (BO) methods struggle with the scalability required for high-dimensional input spaces, we demonstrate how leveraging the mean-field assumption on the black-box function can transform BO into an efficient and scalable solution. Specifically, we introduce MF-GP-UCB, a novel efficient algorithm designed to optimise agent payoffs in this setting. Our theoretical analysis establishes a regret bound for MF-GP-UCB that is independent of the number of agents, contrasting sharply with the exponential dependence observed when naive BO methods are applied. We evaluate our algorithm on a diverse set of tasks, including real-world problems, such as optimising the location of public bikes for a bike-sharing programme, distributing taxi fleets, and selecting refuelling ports for maritime vessels. Empirical results demonstrate that MF-GP-UCB significantly outperforms existing benchmarks, offering substantial improvements in performance and scalability, constituting a promising solution for mean-field, black-box optimisation. The code is available at \url{https://github.com/petarsteinberg/MF-BO}. \looseness=-1

\end{abstract}

\section{Introduction}

\begin{figure}[t!]
    \centering
    \includegraphics[width=0.48\textwidth,trim={0 1.2cm 4.2cm 1cm},clip]{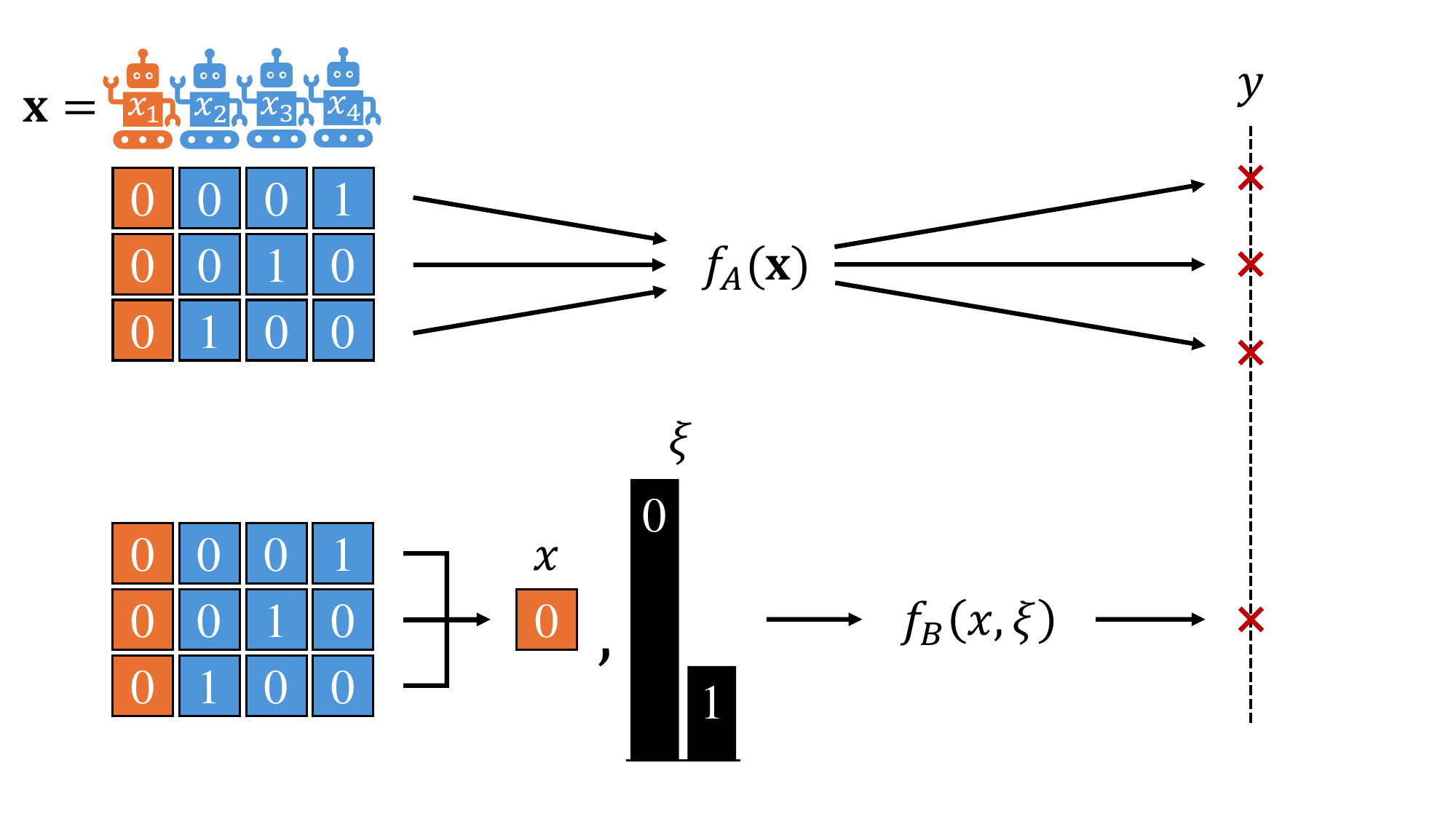}
    \vspace*{\dimexpr-\baselineskip\relax} 
    \caption{MF-GP-UCB utilises invariance under permutation of actions to achieve a regret bound independent of the number of agents. Each row on the left represents a set of actions of the four agents, with the action of the representative agent (RA) in orange and the others in blue. \textbf{Top:} For a fixed RA action $x_1$, the function $f_A$ maps distinct permutations of action vector $\mathbf{x}$ to different values. \textbf{Bottom:} The function $f_B$ utilises the mean-field assumption, which converts permutations of action vector $\mathbf{x}$ into an identical distribution $\xi$, to output a single value for the RA given the observed distribution.}
    \label{fig:mfbo_invariance}
\end{figure}

Bayesian optimisation (BO) is a powerful tool for optimising expensive black-box functions. However, its application to large-scale problems with high-dimensional inputs has been limited due to computational and statistical challenges.
On the other hand, multi-agent systems involving a large population of cooperative agents operated by a central controller have recently received a lot of attention within the control and reinforcement learning communities due to the surge of mean-field algorithms, which show promising signs of tackling scalability issues. Mean-field control (MFC) and mean-field multi-agent reinforcement learning (MF-MARL) were successfully applied in cooperative settings, including taxi repositioning \citep{jusup2023safe,wang2020joint} and ride-sharing order dispatching \citep{li2019efficient}, but we observe an even bigger surge in mean-field games that include ride-sourcing \citep{salhab2017dynamic,zhang2023ride}, traffic routing \citep{courcoubetis2023stationary}, power markets \citep{bichuch2024stackelberg}, and carbon emissions regulation \citep{dayanikli2024multi}.
An interesting stream of research also analyses theoretical properties of mean-field algorithms that include cooperative, competitive, discrete and continuous formulations \citep{hu2023graphon,lasry2007mean,chen2021pessimism,Bauerle2021MeanProcesses,carmona2019model,gast2012mean,Gu2019DynamicMFCs,Gu2021Mean-FieldAnalysis,Motte2019Mean-fieldControls,pasztor2021efficient}. Still, this promising line of work was done through the lens of the mean-field distribution over state or state-action space, while a stateless setting that assumes the mean-field distribution over actions is yet to be explored.
This research gap complements a need for addressing BO scalability challenges. Thus, we exploit the mean-field assumption to introduce significant algorithmic advancements in high-dimensional BO, yield associated theoretical insights, and show this synergy is a versatile tool for a wide range of practical problems. 
To our knowledge, we are the first to introduce an algorithm for the bandit setting and showcase its usefulness in applications such as bike-sharing, ride-sharing, and maritime refuelling. \looseness=-1

We consider a cooperative multi-agent system involving a large population of identical agents interacting within an unknown game characterized by bandit feedback.  The unknown game refers to a setting where each agent can query the environment by executing its action. In turn, the agent receives a reward and the empirical distribution of other agents' actions. Being identical, all agents follow a policy assigned by the global controller at each time step.
The objective is to design an efficient algorithm that maximises the average reward across the population. As a motivating example, one can consider vehicle repositioning, where agents (e.g., taxis or vessels) are assigned to resources (e.g., passengers or ports). Then, for example, each agent seeks to maximise utility (e.g., number of rides) or minimise costs (e.g., waiting time). A greedy resource acquisition can negatively impact overall performance by, for example, causing congestion if most taxis are dispatched to the neighbourhoods with the highest demand or if too many vessels decide to refuel in the port with the best infrastructure. Optimal system performance emerges from a balanced distribution of agents across available resources, effectively mitigating congestion. In \Cref{sec:experiments}, we analyse a range of problems relevant to practitioners where cooperation is necessary for increased utility. \looseness=-1

The introduced setting can also be viewed as high-dimensional BO with the goal of optimising the objective over a set of high-dimensional inputs. In general, the regret bound scales exponentially with the number of dimensions. To improve the regret bounds, many structural properties were imposed on the shape of the reward function (e.g., an additive decomposition as investigated by \citet{rolland2018high}, \citet{han2021high}, and \citet{ziomek2023random}). Various methods have also been introduced to achieve empirical speedups (e.g., projections into lower-dimensional space by \citet{wang2016bayesian} and \citet{nayebi2019framework}, or reducing the search space by defining a trust region from \citet{eriksson2019scalable}). 
The major advantage of introduced MFBO is its invariance to the order of observed actions as depicted in \Cref{fig:mfbo_invariance}, which consequently makes the regret bounds independent of the number of agents. 
In this work, we answer the following question: 

\vspace{0.2cm} 
\textit{Can we design an algorithm with a regret bound independent of the number of agents for a cooperative multi-agent setting with bandit style feedback?}

\begin{tcolorbox}[
colback=green!5!white,
		colframe=black,
		arc=4pt,
		boxsep=0.3pt,
	]%
	\textbf{Contribution 1:} We propose MF-GP-UCB -- the first algorithm for Bayesian optimisation that leverages \textit{the mean-field assumption}.
\end{tcolorbox}%
\begin{tcolorbox}[
colback=blue!5!white,
		colframe=black,
		arc=4pt,
		boxsep=0.3pt,
	]
	\textbf{Contribution 2:} We prove the regret bound for MF-GP-UCB and show it is \emph{independent of the number of agents}.
\end{tcolorbox}%
\begin{tcolorbox}[
colback=red!5!white,
		colframe=black,
		arc=4pt,
		boxsep=0.3pt,
	]
	\textbf{Contribution 3:} We empirically demonstrate the superiority of our algorithm on a number of synthetic and real-world problems.
\end{tcolorbox}%

\section{Related Work}
\paragraph{High-dimensional BO}
While our work focuses on the setting where a large population of agents cooperates to optimise an unknown black-box function, we can also think about this problem from the perspective of high-dimensional BO, where we wish to optimise a function over a set of high-dimensional inputs. Standard BO struggles in high-dimensional spaces and, in general, the regret bound of the standard GP-UCB algorithm scales exponentially with the number of dimensions \cite{srinivas2009gaussian, chowdhury2017kernelized}. A number of different algorithms have been presented to tackle this problem. Methods such as REMBO \cite{wang2016bayesian} or HeSBO \cite{nayebi2019framework} project the original space into a space of lower dimensionality and conduct the optimisation process there. Decomposition methods \cite{rolland2018high, han2021high, ziomek2023random} assume the function can be additively decomposed into subfunctions, operating on spaces of lower dimensionalities. SaaSBO \cite{eriksson2021high} removes the dimensions that do not seem to be impacting the function value too much, thus reducing the difficulty of the problem. The famous TuRBO \cite{eriksson2019scalable} defines a trust region to which the optimisation process is restricted, effectively reducing the volume of the search space. Within this work, we exploit certain invariant properties of mean-field functions to achieve higher sample efficiency. The recent work of \citet{brown2024sample} studied invariances in BO, in general, but without considering mean-field systems specifically. Our mean-field assumption could be considered a permutation-invariance from their work's point of view. However, applying their result directly to our case will result in a regret bound scaling as $\mathcal{O}(\beta_T\sqrt{T\log T^M} / M!)$, which is better than naive BO, but still exhibits dependence on $M$, especially for large $T$. In contrast, the bound of our algorithm is completely independent of $M$. 
\looseness=-1

\paragraph{Mean-Field Control and MARL}
Closest to our setting are mean-field control (MFC) and mean-field multi-agent reinforcement learning (MF-MARL) formulations. We refer to MFC as the setting with a large population of cooperative agents under a known environment, while we reserve MF-MARL for an unknown environment.
\citet{gast2012mean} are the first to analyse MFC as a Mean-Field Markov Decision Process (MF-MDP) and show that the optimal reward converges to the solution of a continuous Hamilton-Jacobi-Bellman equation under some conditions.
\citet{Motte2019Mean-fieldControls} introduce MF-MDP under mean-field interaction both on state and actions and show the existence of $\epsilon$-optimal policies. 
\citet{carmona2019model} define MFC as an MF-MDP  over a limiting distribution of continuous agents' states and show an optimal policy exists. They further introduce a discretisation strategy for MFQ-learning.  
\citet{Bauerle2021MeanProcesses} formulate MFC as an MF-MDP where reward and transition functions depend on the empirical measure of the agents' states and show the existence of $\epsilon$-optimal policy under some conditions. 
\citet{chen2021pessimism} introduce and analyse a sub-optimality gap of SAFARI--an offline MF-MARL algorithm--for settings where the interaction with the environment during training can be prohibitive or even unethical (e.g., social welfare or other societal systems).
\citet{Gu2019DynamicMFCs,Gu2021Mean-FieldAnalysis} first set dynamic programming principles for MFC and then show that model-free kernel-based Q-learning has a linear convergence rate for MFC. They further show that the MFC approximation of cooperative MARL has an approximation error $\bigO(\nicefrac{1}{\sqrt{N}})$ with the number of agents $N$. \citet{hu2023graphon} show that a graphon MFC also achieves the discussed approximation error. 
\citet{pasztor2021efficient} introduce M\textsuperscript{3}-UCRL--an algorithm with a sublinear cumulative regret--for MF-MARL where the goal is to simultaneously optimise for the rewards and learn the dynamics from the online experience.
\citet{jusup2023safe} build on top of \citet{pasztor2021efficient} and introduce Safe-M\textsuperscript{3}-UCRL, an online algorithm for constrained MF-MARL that learns underlying dynamics while satisfying constraints throughout the execution. \looseness=-1

\section{Problem Statement} \label{sec:problem_statement}

We consider a problem setting, where $M$ agents cooperate to maximise a payoff of an unknown game. At each timestep $t$, each agent is assigned a context $\bm{c}_t \in C \subset \mathbb{R}^{d_c}$ with probability defined by some measure $p(\bm{c})$ and  they can choose a discrete action  $\bm{x}_t \in A \subset \mathbb{R}^{d_a}$, such that $|A| << M$. This context $\bm{c}$ can, for example, describe the agent's type or its objective at the current timestep $t$. For each $\bm{c}_t \in C$, we denote by $\xi_t(\cdot|\bm{c}) \in \Delta_A$ the distribution of selected actions for players with context $\bm{c}$ at time $t$ and we write $\xi_t = (\xi_t(\cdot|\bm{c}))_{\bm{c} \in C}$ to mean the collection of such distributions for each context value $\bm{c}$. As such, $\xi_t \in \Delta^C_A$, where $\Delta_A^C = \times_{\bm{c} \in C} \Delta_A$ denotes a set of joint distributions over actions for players with each context value. We assume that the payoff of the game is defined by some unknown function $f: A \times C \times  \Delta_A^C \to \mathbb{R}$. The problem setting of an unknown game with bandit-style feedback has been previously proposed by \citet{sessa2019no}. In their approach, the goal of the $i$-th agent is to maximise its unknown payoff function $ f^i(\bm{x}^i_t, \bm{x}^{-i}_t)$.
In this work, we assume a mean-field control setting (i.e., agents are identical and cooperate) and that each agent is given some context $\bm{c} \in C$, describing its type in a given round.  This means that $f^i(\bm{x}^i_t, \bm{x}^{-i}_t) \approx f(\bm{x}^i_t, \bm{c}_t^i, \xi_t)$ for each $i=1,\dots, M$. As we assume a mean-field setting, all agents are identical and have to follow the same procedure for selecting their actions.  After each round, the representative agent observes its own, noisy payoff $y_t = f(\bm{x}, \bm{c}, \xi_t) + \epsilon_t$, corrupted by some Gaussian noise $\epsilon_t \sim \mathcal{N}(0, R^2)$ with known variance $R^2$. We then wish to maximise the average population payoff $\mathbb{E}_{(\bm{x}, \bm{c}) \sim \xi_t(\bm{x}|\bm{c})p(\bm{c})}[f(\bm{x}, \bm{c}, \xi_t)]$. As such, the regret of the population is defined as:
\begin{align*}
     R_T =&  \sum_{t=1}^T\mathbb{E}_{(\bm{x}, \bm{c}) \sim \xi^\star(\bm{x}|\bm{c})p(\bm{c})}[f(\bm{x}, \bm{c}, \xi^*)] \\
     & - \sum_{t=1}^T\mathbb{E}_{(\bm{x}, \bm{c}) \sim \xi_t(\bm{x}|\bm{c})p(\bm{c})}[f(\bm{x}, \bm{c}, \xi_t)],
\end{align*}
where 
$
    \xi^* = \arg\max_{\xi \in \Delta_A^C} \mathbb{E}_{(\bm{x}, \bm{c}) \sim \xi(\bm{x}|\bm{c})p(\bm{c})}[f(\bm{x}, \bm{c}, \xi)]
$ is the optimal agents' action distribution. 
In other words, the regret is the difference between the average population payoff while using the optimal distribution $\xi^\star$ and the payoff when using the selected distribution $\xi_t$.  

\begin{table}[t!]
    \centering
    \caption{Comparison of regret bounds obtained with standard GP-UCB and our proposed algorithm, MF-GP-UCB, when utilising the squared exponential kernel in the proposed problem setting. By $\iota$ we denote $\log T$.} \label{tab:regret_comparison}
    \vspace*{\dimexpr\baselineskip\relax} 
    \begin{tabular}{c|c}
        Algorithm & Regret Bound \\
        \midrule
        GP-UCB  & $\mathcal{O}(\beta_T \sqrt{T \iota^{Md_A + d_C}})$\\ [0.1cm]
        Add-GP-UCB  & $\mathcal{O}(\beta_T \sqrt{M T \iota^{d_A + d_C}})$\\ [0.1cm]
        MF-GP-UCB & $\mathcal{O}(\beta_T \sqrt{T \iota^{d_A  + d_C + |A||C|}})$ \\ [0.1cm]
        \makecell{MF-GP-UCB \\ (Additive)} & $\mathcal{O}\left(\beta_T \sqrt{T \left(|C|\iota^{|A|} + \iota^{d_C} + \iota^{d_A }\right)}\right)$
    \end{tabular}
\end{table}

\textbf{Assumptions on the black-box function} As mentioned before, we assume that the black-box function obeys the mean-field assumption, which is that the payoff of agent $i$ at step $t$ depends only on its own action $\bm{x}_t^i$, its context $\bm{c}_t^i$ and the distribution of other agents' actions $\xi_t$. Formally, we assume that the blackbox function $f$ is a sample from a Gaussian Process prior \cite{rasmussen2003gaussian} $\mathcal{GP}(0, k(\cdot, \cdot))$ equipped with a  kernel function $k((\bm{x}, \bm{c}, \xi),(\bm{x}^\prime, \bm{c}^\prime, \xi^\prime))$, which we assume is known. This kernel operates on multiple arguments, one being a collection of probability distribution for each context. The distribution of actions for each context can be thought of as a vector in $\mathbb{R}^{|A|}$ (where $A$ is the set of possible actions).
We make a standard \cite{srinivas2009gaussian} assumption regarding the kernel smoothness with respect to this continuous input:

\begin{assumption} \label{as:lkernel}
    The kernel $k(\cdot, \cdot)$ satisfies the following condition on the derivatives of a sample path $f \sim \mathcal{GP}(0, k)$. There exist constants $a, b > 0$, such that:
\[
P \left(
\max_{\bm{x} \in A} \max_{\bm{c} \in C}\sup_{\xi \in \Delta^C_A} \left\lvert \frac{\partial f}{\partial \xi(\bm{x}^\prime|\bm{c}^\prime)} \right\rvert  > L
\right)
\leq a \exp \left( - \frac{L^2}{b} \right),
\]
for each $\bm{x}^\prime \in A$ and $\bm{c}^\prime \in C$.
\end{assumption}

\section{Proposed Algorithm}
Analogously to standard BO, we would first like to fit a Gaussian Process model (GP) to the observations collected so far $\mathcal{D}_{t-1}$. For brevity, let us denote the triple $(\bm{x}, \bm{c}, \xi)$ as $\bm{z}$. Given a kernel function $k(\bm{z}, \bm{z}^\prime)$ and the data obtained so far $\mathcal{D}_{t-1} = \{(\bm{z}_\tau, y_\tau)\}_{\tau=1}^{t-1}$, we can utilise standard GP update equations \cite{rasmussen2003gaussian} to obtain mean $\mu_{t-1}(\bm{z})$ and variance $\sigma^2_{t-1}(\bm{z})$ functions, as given below:
\begin{align*}
    \mu_{t-1}(\bm{z}) &= \bm{k}^T_{t-1} (\bm{K}_{t-1} + R^{-1}\mathbf{I})^{-1} \bm{y}\\
    \sigma^2_{t-1}(\bm{z}) &= k(\bm{z}, \bm{z}) - \bm{k}^T_{t-1} (\bm{K}_{t-1} + R^{-1}\mathbf{I})^{-1}\bm{k}_{t-1},
\end{align*}
where $\bm{y} \in \mathbb{R}^{t-1}$ with elements $(\bm{y})_j = y_j$ and $\bm{K}_{t-1} \in \mathbb{R}^{(t-1)\times(t-1)}$ with elements $(\bm{K}_{t-1})_{j,m} = k(z_j, z_m)$ and $\bm{k}_{t-1} \in \mathbb{R}^{t-1}$ with elements $(\bm{k}_{t-1})_j = k(\bm{z}, z_j)$. Similar to standard BO, we would like to rely on some acquisition function to help us select the distribution of actions $\xi$ to try next. We propose to use the following criterion, which we coin Mean-Field-Upper-Confidence-Bound (MF-UCB):
\begin{align*}
    \alpha_t(\xi|\mathcal{D}_{t-1}) = \mathbb{E}_{(\bm{x}, \bm{c}) \sim \xi(\bm{x}|\bm{c})p(\bm{c})}[\mu_{t-1}(\bm{z}) +\beta_t \sigma_{t-1}(\bm{z})]].
\end{align*}
 In \Cref{alg:ucb_mf_bo_centralised}, we present MF-GP-UCB, an algorithm that employs this acquisition function. At each timestep $t$, the algorithm fits a GP model to the data gathered so far, and then in line 5, it selects the context-dependent distribution of actions by optimising the MF-UCB criterion. Note that this criterion requires us to optimise over a distribution with respect to which the expectation is computed. In practice, this could be easily done with a reparametrisation trick \cite{kingma2015variational}. We now proceed to derive the bound on the representative agent's regret when using the MF-GP-UCB algorithm.

 \begin{theorem} \label{thm:ucb_mf_bo_bound_centralised}
    Let Assumption \ref{as:lkernel} hold, and run \Cref{alg:ucb_mf_bo_centralised} for $T$ rounds and set  $\beta_t = 2\log(|A||C||\Xi_t|t^2 / \sqrt{2\pi})$, where:
    \begin{align*}
        |\Xi_t| = \left( b|A||C|t^2(\log(a|A||C|) + \sqrt{\pi} / 2) \right)^{|A||C|}
    \end{align*}
 then we have that:
    \begin{equation*}
       \mathbb{E}[R_T] \le \mathcal{O}(\beta_T \sqrt{T \gamma_T} + \mathcal{B}),
    \end{equation*}
    where $\mathcal{B} = \sqrt{2|A|^3|C|}\min\{(b(\log(a|A||C|) + \sqrt{\pi} / 2))^{-1},|A||C|\}$ and $\gamma_T$ is the maximum information gain of the kernel.
\end{theorem}
\begin{proof} 

We provide a sketch here and defer the full proof to
Appendix \ref{ap:ucb_mf_bo_bound_centralised_proof}.    
Our setting poses an additional layer of difficulty as the actions $\bm{x}_t$ of the representative agents are not deterministic, even conditioned on all previous data. We thus follow the general proof idea of \citet{takeno2023randomized}, but with important modifications. First of all, for the purpose of analysis, we introduce a discretisation of the space of possible distributions $\Xi_t$ that becomes finer with each iteration $t$. Let us define $g(\xi) = \mathbb{E}_{\bm{x},\bm{c} \sim \xi(\bm{x}|\bm{c})p(\bm{c})}[f(\bm{x}, \bm{c}, \xi)]$. We can then decompose regret as:
\begin{align*}
    \mathbb{E}[r_t] &= \mathbb{E}[g(\xi^\star) - g([\xi^\star]_t)] + \mathbb{E}[g([\xi^\star]_t)  -
    g(\xi_t)].
\end{align*}
Bounding the term $g([\xi^\star]_t)  - g(\xi_t)$ is relatively easy, as by carefully manipulating the expectations, we can show that $\alpha_t(\xi|\mathcal{D}_{t-1})$ is a high probability upper bound on  $\mathbb{E}[g(\xi_t)]$ and the rest follows the standard expected regret BO proof on finite domain $\Xi_t \times A \times C$.

The main difficulty lies in bounding $\mathbb{E}[g(\xi^\star) - g([\xi^\star]_t)]$. This difference can be decomposed as:
\begin{align*}
    g(\xi^\star) &- g([\xi^\star]_t) =\\& \mathbb{E}_{\bm{x},\bm{c} \sim \xi^\star(\bm{x}|\bm{c})p(\bm{c})}[f(\bm{x}, \bm{c}, \xi^\star) - f(\bm{x}, \bm{c}, [\xi^\star]_t)] \\ +\,& \mathbb{E}_{\bm{c} \sim p(\bm{c})}\left[\sum_{\bm{x} \in A}f(\bm{x}, \bm{c}, [\xi^\star]_t)([\xi^\star]_t(\bm{x}|\bm{c}) - \xi^\star(\bm{x}|\bm{c})  )\right].
\end{align*}
In other words, it is equal to the difference in function $f(\cdot)$ evaluated for $\xi^\star$ and $[\xi^\star]_t$ for a fixed distribution of $\bm{x}$, plus the difference in those sampling distributions. Bounding the first term can be achieved thanks to \Cref{as:lkernel}, ensuring function continuity. Bounding the second term can be achieved by bounding the expected maximum of the Gaussian Process and thanks to the fact that the discretisation is getting finer and thus $[\xi^\star]_t(\bm{x}|\bm{c}) - \xi^\star(\bm{x}|\bm{c})$ smaller.
\end{proof}

\begin{algorithm}[t!]
   \caption{MF-GP-UCB}
   \label{alg:ucb_mf_bo_centralised}
\begin{algorithmic}[1]
   \STATE {\bfseries Input:} action space $A$, evaluation budget $T$, exploration bonuses $\{\beta_t\}_{t=1}^T$, kernel function $k(\bm{x}, \bm{c}, \xi)$
   \STATE Initialize $\mathcal{D}_0 = \emptyset$.
   \FOR{$t=1$ {\bfseries to} $T$}
   \STATE Fit GP to $\mathcal{D}_{t-1}$, obtaining $\mu_{t-1}$, $\sigma_{t-1}$
   \STATE Solve \begin{equation*}
        \xi_t = \arg \max_{\xi \in \Delta^C_A}\alpha_t(\xi|\mathcal{D}_{t-1})  
    \end{equation*} \label{alg_line:acq_opt}
    \STATE  The population selects their actions according to $\xi_t$
    
   \STATE Representative agent observes the function value $y_t= f(\bm{x}_t, \bm{c}_t, \xi_t) + \epsilon_t$ 
   \STATE Data-buffer is updated $\mathcal{D}_{t} = \{(\bm{x}_t, \xi_t, \bm{c}_t, y_t)\} \cup \mathcal{D}_{t-1}$ 
   \ENDFOR
\end{algorithmic}
\end{algorithm}

We note that the derived result depends heavily on the maximum information gain (MIG) property of the used kernel function $k(\cdot, \cdot)$. If we were to utilise the RBF kernel defined below: \looseness=-1
\begin{equation*}
    k_\text{RBF}(\bm{z}, \bm{z}^\prime) = \exp\left( - \frac{\lVert \bm{z} - \bm{z}^\prime \rVert_2}{2 l} \right), 
\end{equation*}
for $\bm{z}, \bm{z}^\prime \in \mathbb{R}^d$ and some lengthscale $l > 0$, the standard result \citet{srinivas2009gaussian} establishes that $\gamma_T \le \mathcal{O}((\log T )^d)$. Due to the fact that our kernel operates on the distribution of actions rather than the actions themselves, the MIG of the kernel will not depend on the number of agents. If we represent the distribution for all contexts as a vector in $|A||C|$-dimensional space, $\bm{x} \in \mathbb{R}^{d_A}$ and $\bm{c} \in \mathbb{R}^{d_C}$, then the joint input $\bm{z} \in \mathbb{R}^{d_A + d_C + |A||C|}$. Running MF-GP-UCB utilising a joint RBF kernel over all inputs $\bm{z}$ results in the regret bound of $\mathcal{O}(\beta_T \sqrt{T \iota^{d_A  + d_C + |A||C|}})$, where $\iota = \log T$. Note that a naive application of standard BO to the input space of actions of $M$ agents and the context would result in a bound scaling exponentially with $Md_A + d_C$ instead. As such, our algorithm provides an improvement as long as $d_A  + |A||C| \le Md_A$ and provides significant improvement when $M$ is extreme compared to other quantities. We summarise these results in \Cref{tab:regret_comparison}, where we also include a popular variant of GP-UCB called Add-GP-UCB \citep{kandasamy2015high, rolland2018high} that employs a kernel that is a sum of subkernels operating on smaller dimensions. For such a kernel the MIG was shown \citep{rolland2018high} to grow as $\mathcal{O}(\sum_{i=1}^n(\log T)^{d_i})$, assuming the sum is over $n$ subkernels and $i$-th subkernel operates on $d_i$ dimensions. Note that even though the application of an additive kernel reduces the order dependence on $M$ from exponential to linear (as then $n=M$), it is still worse than the bound of MF-GP-UCB, which is unaffected by $M$. In MF-GP-UCB we can also employ additive kernels for an additional reduction in the MIG. To achieve greater sample efficiency, within our experiments, we utilised such an additive kernel:
\begin{equation} \label{eq:used_kernel}
    k(\bm{z},\bm{z}')=k_\text{RBF}(\bm{x}, \bm{x}')+k_\text{RBF}(\bm{c}, \bm{c}')+\sum_{\bm{c}\in C}k_\text{RBF}(\xi_c, \xi_c'),
\end{equation}
and we also included its regret bound in the aforementioned table for comparison under the name MF-GP-UCB (Additive). We now proceed to present the empirical performance of MF-GP-UCB, equipped wih such a kernel function.

\section{Experiments} \label{sec:experiments}

We divide our experiments into two categories: i) synthetic toy problems and ii) real-world problems, and we are interested in evaluating algorithm sample efficiency and solution quality.
For benchmark comparisons, we use TuRBO \cite{eriksson2019scalable} to represent the state-of-art algorithm for high-dimensional BO, alongside combinatorial optimisers such as Simulated Annealing (SA) and Genetic Algorithm (GA). We also include the random algorithm as a baseline. We utilise three publicly available datasets to develop custom MFC-inspired black-box functions that resemble real-world problems. Note that the number of agents $M$ is synonymous and used interchangeably with the dimensionality of a problem. Similarly, the term reward/payoff is often used to refer to the black-box function value. Running the experiments required approximately 100 days of CPU time on a cluster with varying CPU models, 93\% of which was used by TuRBO and about 2\% by each of the remaining algorithms (excluding random search). In \Cref{tab:runtimes}, we report a more detailed breakdown of the runtime (in minutes) for each experiment and algorithm.

\begingroup
\renewcommand{\arraystretch}{2}
\begin{table}[t!]
\caption{Runtime mean and standard deviation over 10 runs, reported in minutes for each experiment and algorithm. Random search is omitted as it has a negligible runtime.} \label{tab:runtimes}
\vspace*{\dimexpr\baselineskip\relax} 
\centering
\scriptsize
\begin{tabular}{l|c|c|c|c|}
\cline{2-5}
                                & \makecell{Genetic\\Algorithm} & \makecell{Simulated\\Annealing} & TuRBO & \makecell{MF-GP-\\UCB} \\ \hline
\multicolumn{1}{|l|}{Swarm}     &         $0.09\pm0.01$          &       $0.12\pm0.01$              &   $9.3\pm0.2$    &      $14\pm1$     \\ \hline
\multicolumn{1}{|l|}{Arena}     &     $0.11\pm0.01$           &        $0.14\pm0.01$             &   $10\pm1$    &    $22\pm2$       \\ \hline
\multicolumn{1}{|l|}{LouVelo}   &       $0.64\pm0.25$            &        $0.82\pm0.32$             &    $50\pm13$   &     $14\pm2$      \\ \hline
\multicolumn{1}{|l|}{Maritime}  &      $7.2\pm2.6$             &        $9.3\pm3.4$             &    $475\pm113$   &     $54\pm3$      \\ \hline
\multicolumn{1}{|l|}{NYC (a)} &       $88\pm13$            &         $108\pm16$            &    $4105\pm477$   &    $36\pm2$       \\ \hline
\multicolumn{1}{|l|}{NYC (b)} &        $95\pm4$           &        $118\pm4$             &   $4557\pm470$    &     $62\pm3$      \\ \hline
\multicolumn{1}{|l|}{NYC (c)} &          $89\pm15$         &        $113\pm19$             &   $4206\pm596$    &     $131\pm7$      \\ \hline
\end{tabular}
\vspace{-0.5cm}
\end{table}
\endgroup

\begin{figure*}[t!]
    \centering
    \subfigure[\textbf{Swarm motion:} $M=50$, $|A|=30$, $|C|=1$.]{%
        \includegraphics[width=0.24\textwidth, trim={0 0 1.7cm 2.5cm}, clip]{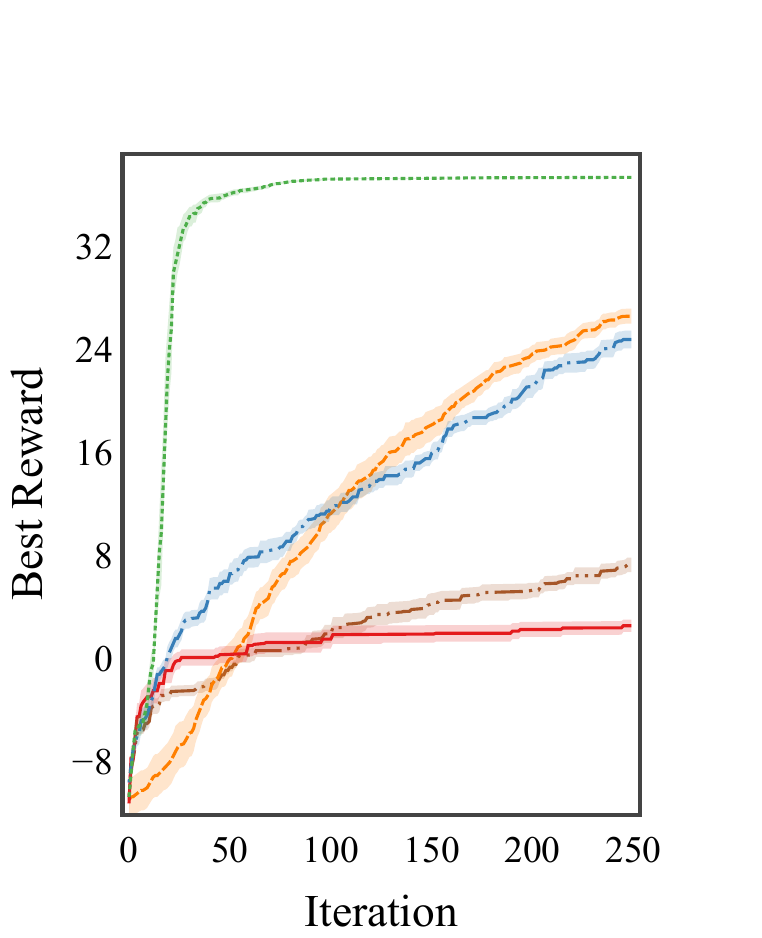}
        \label{fig:swarm}
    }%
    \subfigure[\textbf{Arena:} $M=50$, $|A|=30$, $|C|=2$.]{%
        \includegraphics[width=0.24\textwidth, trim={0 0 1.7cm 2.5cm}, clip]{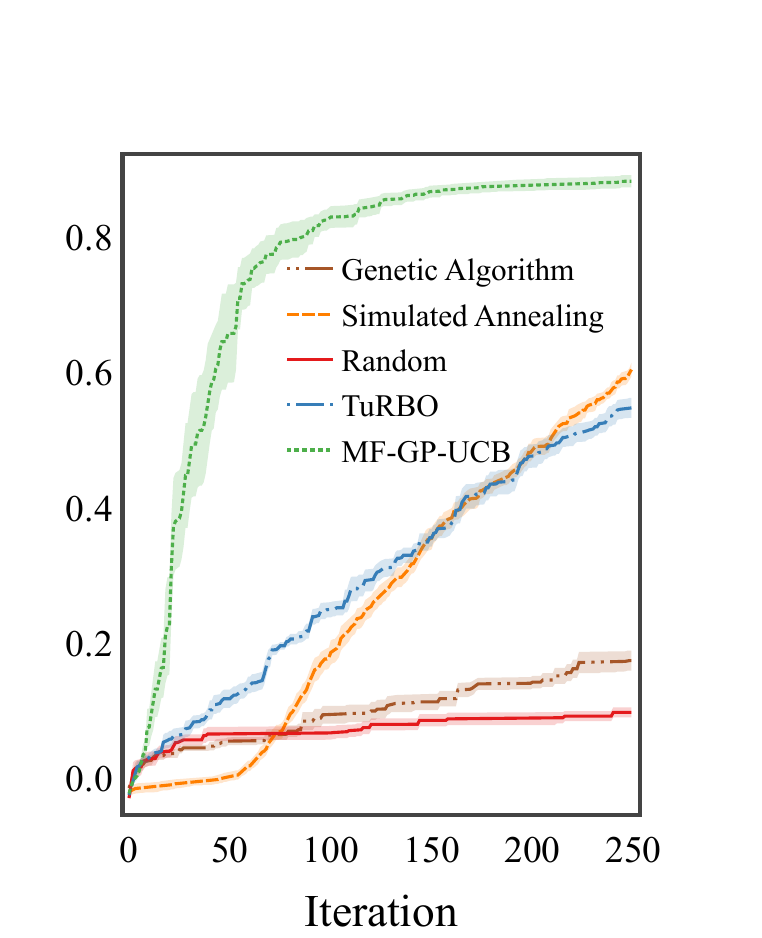}
        \label{fig:arena}
    }%
    \subfigure[\textbf{LouVelo bike-sharing:} $M=300$, $|A|=36$, $|C|=1$.]{%
        \includegraphics[width=0.24\textwidth, trim={0 0 1.7cm 2.5cm}, clip]{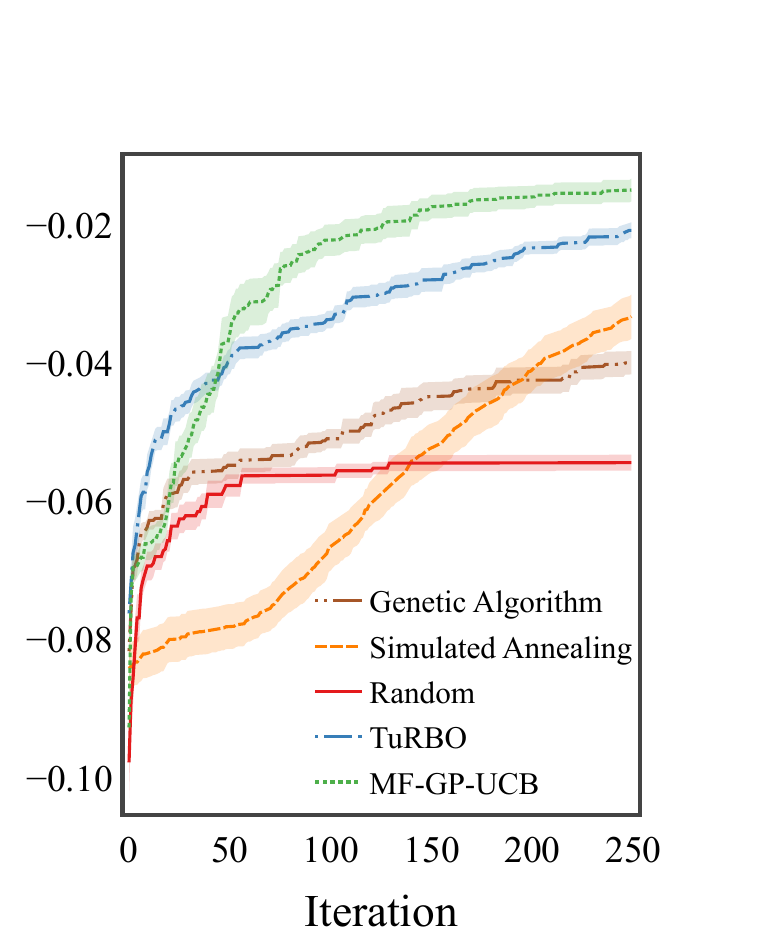}
        \label{fig:louvelo}
    }%
    \subfigure[\textbf{Maritime refuelling:} $M=3000$, $|A|=30$, $|C|=5$.]{%
        \includegraphics[width=0.24\textwidth, trim={0 0 1.7cm 2.5cm}, clip]{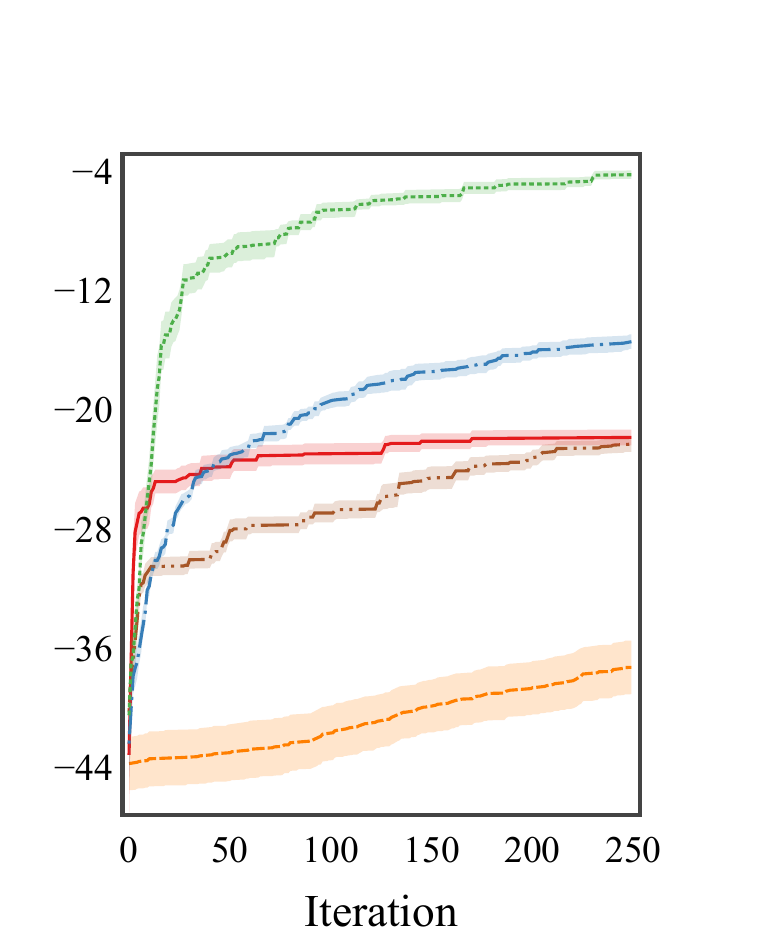}
        \label{fig:maritime}
    }%
    \caption{MF-GP-UCB is superior in both sample efficiency and solution quality compared to the benchmarks over a range of black-box dimensions $M$. When the black box satisfies the mean-field assumption, our algorithm inevitably has the advantage by optimising over the \textit{distribution} of actions instead of over the interactions of individual actions.}
    \label{fig:results}
\end{figure*}

\subsection{Implementation} \label{exp:implementation}
The implementations for the benchmarks were adopted from the MCBO framework \cite{dreczkowski2024framework}.
Implementing MF-GP-UCB involved the use of BoTorch \cite{balandat2020botorch} and GPyTorch \cite{gardner2018gpytorch} for the GP model with an additive kernel composed as defined in Equation \ref{eq:used_kernel},
and the Adam algorithm \cite{kingma2014adam} from PyTorch \cite{paszke2019pytorch} was used for optimising the distribution over actions in \Cref{alg:ucb_mf_bo_centralised} (line 5). We used the default parameter values, with the exception of a slightly higher learning rate $\gamma=0.01$. Context sampling between iterations was enabled with instances of MF-GP-UCB, while for the benchmarks, we assigned the contexts at the start of the optimisation and kept them fixed. This was done to simplify implementation since the context is ideally randomly assigned, but the benchmarks are designed for full control over the inputs. The effect of this on benchmark performance is only advantageous since the agents/dimensions don't switch context. For clarity, the result plots display the \textit{best} reward by the respective algorithms as opposed to the actual reward obtained at a given iteration. For instance, if an algorithm regresses and begins to suggest worse solutions, its previously best found reward value will accumulate. We average these values over 10 seed runs for $T=250$ iterations, with shaded areas representing the standard error. The random algorithm generates a random uniform vector in $[0, 1]^{|A|}$ for each context before passing it through a Softmax and sampling actions from it.

\subsection{Synthetic Experiments} \label{exp:synthetic}
Our synthetic toy problems aimed to investigate MF-GP-UCB's performance in low-dimensional settings before scaling up to hundreds and eventually thousands of agents. \looseness=-1

\subsubsection{Swarm Motion} \label{synthetic:swarm}
Our swarm motion problem is a modification of the standard version, such as described previously \cite{jusup2023safe}, to fit into a discrete and stateless problem setting. Each agent picks an action and contributes to the system reward given by:
\[
    r(\mathbf{x})=\underbrace{\sum_{m=1}^M\Bigr [2\pi^2\Bigl(\sin{(x_m)}-\cos^2{(x_m)}\Bigr) + 2\sin{(x_m)}\Bigl]}_\text{reward term}\]
    \vspace{-0.15cm}
    \[\underbrace{-\sum_x\sigma\log{(P_x+1)}}_\text{penalty term},
\] 
where $P_x$ is the normalised frequency of action $x$ over all agents and $\sigma$ is a congestion factor. Simply put, the agents will want to congregate in high-reward areas but will also be punished for creating crowds. The optimal solution is thus a dispersion around the maximum of the reward term at $x=\pi/2$, which is exactly what MF-GP-UCB converges to in about 80 iterations (see \Cref{fig:swarm}). Conversely, if every agent picks the same action, the penalty term will dim the reward term so that such a solution cannot be optimal. The congestion factor $\sigma$ controls the amplitude of this dimming effect -- larger values encourage more dispersion in the optimal distribution. Our configuration uses $M=50$ agents in an action space of $|A|=30$ equally spaced points over $A= [0, 2\pi]$ with a congestion factor $\sigma=10$.

\subsubsection{Arena} \label{synthetic:arena}
The arena function is a custom toy problem designed to showcase support for agent contexts. Each agent is assigned a context by sampling from a predefined measure $p(\bm{c})$. The system is incentivised to avoid crowded areas by a penalty term identical to \Cref{synthetic:swarm}. The total system reward is given by:
\[
    r(\mathbf{x})=\overbrace{\frac{1}{{M \choose 2}}\sum_{\substack{i,j \\i\neq j}}\bm{c}_i\bm{c}_j\cos{(x_i-x_j)}}^\text{reward term} \underbrace{-\sum_x\sigma\log{(P_x+1)}}_\text{penalty term},
\]
considering all interactions between unique pairs of agents $i$ and $j$, with $\bm{c} \in \{-1, 1\}$ and $A= [0, 2\pi]$. An additional optimisation difficulty stems from the fact that the function is multimodal: shifting all the agents by some arbitrary $\theta$ produces the same reward value. We assign $M=50$ agents to pick from $|A|=30$ discrete actions on a circle, with a Rademacher context distribution and a congestion factor $\sigma=10$. The idea is to separate supporters/agents of two opposing teams/contexts in an arena (hence the name) on opposite sides to minimise potential conflict while at the same time avoiding huge crowds. In simple terms, agents pairs with the same/opposite context contribute positively/negatively to the reward if they are close to each other. One such desirable arrangement can be seen in \Cref{fig:arena_histogram}, a solution our algorithm achieves. \Cref{fig:arena} displays the result of the arena experiment, demonstrating the superiority of MF-GP-UCB.

\begin{figure}[t!]
    \centering
    \includegraphics[width=0.36\textwidth]{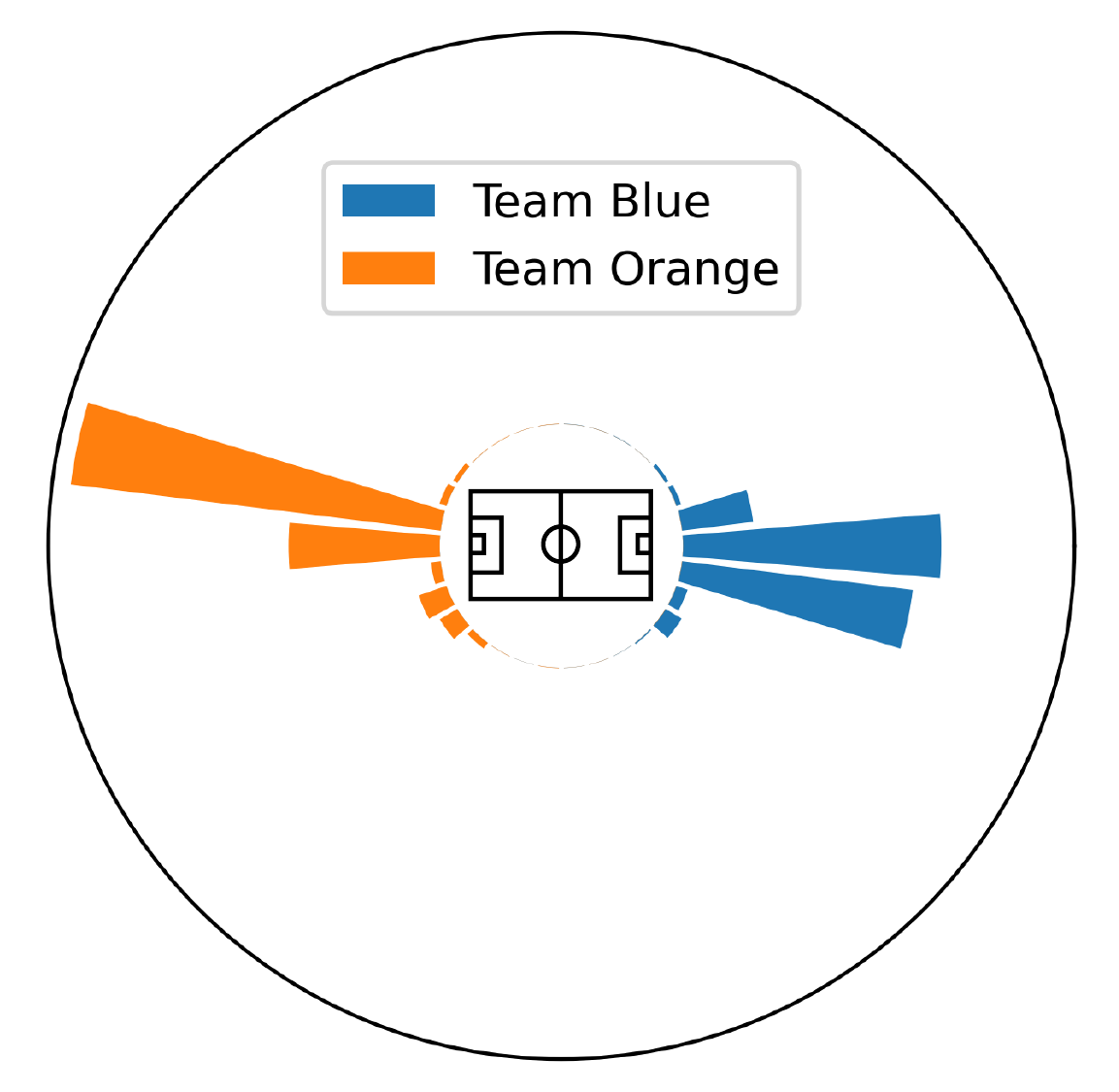}
    \caption{Arena Histogram -- a visual example of a solution suggested by MF-GP-UCB. The reward term encourages separating the supporters of the two teams/contexts around the pitch, while the penalty term ensures there is no extreme congregation in just two directly opposite seating areas. The congestion factor $\sigma$ controls the "smoothness" of the optimal histograms. This solution produced a black-box reward value $r(\mathbf{x})\approx0.9$.}
    \label{fig:arena_histogram}
\end{figure}

\subsection{Real-World Experiments} \label{exp:real-world}

\subsubsection{LouVelo Bike-sharing} \label{real-world:LouVelo}
This dataset \cite{louvelo-data} contains trip-level data of a bike-sharing programme in Louisville, KY. A total of $M=300$ bicycles are borrowed and returned across $|A|=41$ discrete stations distributed over the city. Naturally, the demand for bicycles will vary depending on the station, and thus, it is in the interest of the centralised controller to provide the optimal distribution of bicycles across the city in order to best match the demand. Each agent/bicycle picks an action/station and contributes to the system payoff, computed as the negative of Jensen-Shannon divergence $D_{JS}$ between the empirical distribution of bicycles and the ground truth distribution of demand. $D_{JS}$ for distributions $P$ and $Q$ is given by: 
\[
    D_{JS}(P||Q)=\frac{1}{2}D_{KL}(P||H)+\frac{1}{2}D_{KL}(Q||H),
\] 
where $D_{KL}$ is the Kullback-Leibler (KL) divergence and $H$ is a mixture distribution defined as $H=\frac{1}{2}(P+Q)$. $D_{JS}$ returns $0$ for two identical distributions, making it easy to track proximity to the optimal distribution. We infer a demand distribution from the bicycle pickup data and corresponding timestamps. We chose twenty consecutive Saturdays (as it is the busiest day) to generate twenty normalised distributions, taking into account the pickups throughout the entire day. We expect this to be in line with the bicycle re-positioning that the central controller executes on a daily basis. The mean of the normalised distributions is then taken as the ground truth demand distribution. This experiment is, in essence, a scale-up to the swarm motion experiment, with real data driving the black-box evaluation. It is no surprise then that MF-GP-UCB outperforms the benchmarks, quickly finding a near-optimal solution. The result is displayed in \Cref{fig:louvelo}. \looseness=-1

\subsubsection{NYC Taxi} \label{real-world:NYC}

\begin{figure}[t!]
    \centering
    \includegraphics[width=0.85\columnwidth, trim={0 3.5cm 0.5cm 0}, clip]{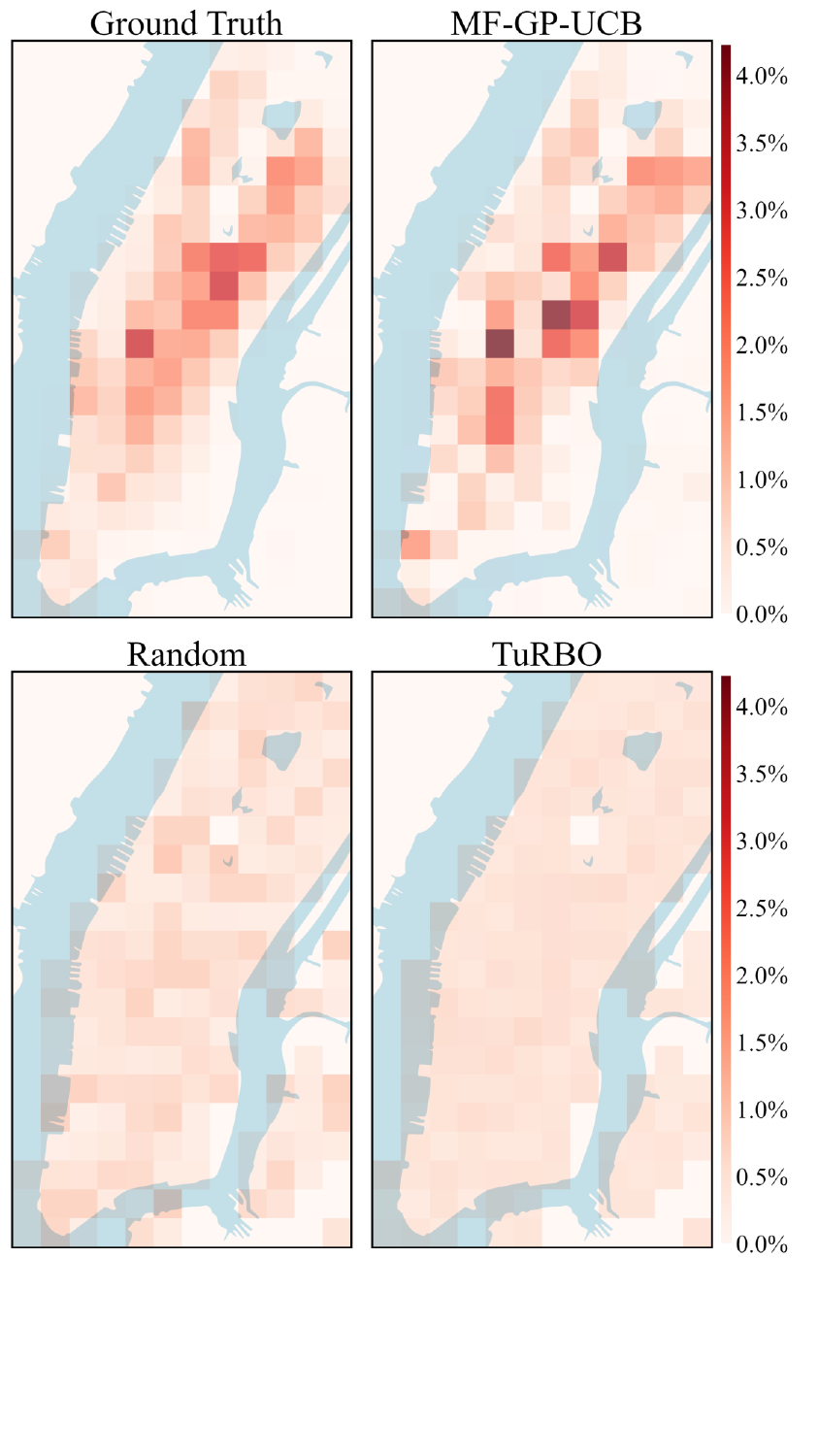}
    \caption{Manhattan Island -- NYC Taxi ground truth demand in a $12\times20$ heatmap and best solutions found by respective algorithms. The solution by MF-GP-UCB is clearly closer to the ground truth than TuRBO, which is closer to random. This is further confirmed in \Cref{fig:NYC-12x20-}. Leveraging the mean-field assumption gives MF-GP-UCB the advantage, while TuRBO is optimising individual actions over $M=20,000$ dimensions and thus not getting much further than uniformity in $T=250$ iterations.}
    \label{fig:NYC-heatmaps}
    
\end{figure}

\begin{figure*}[t!]
    \centering
    \subfigure[grid $=12\times20$, $|A|=160$]{%
        \includegraphics[width=0.32\textwidth, trim={0 0 1.7cm 2.5cm}, clip]{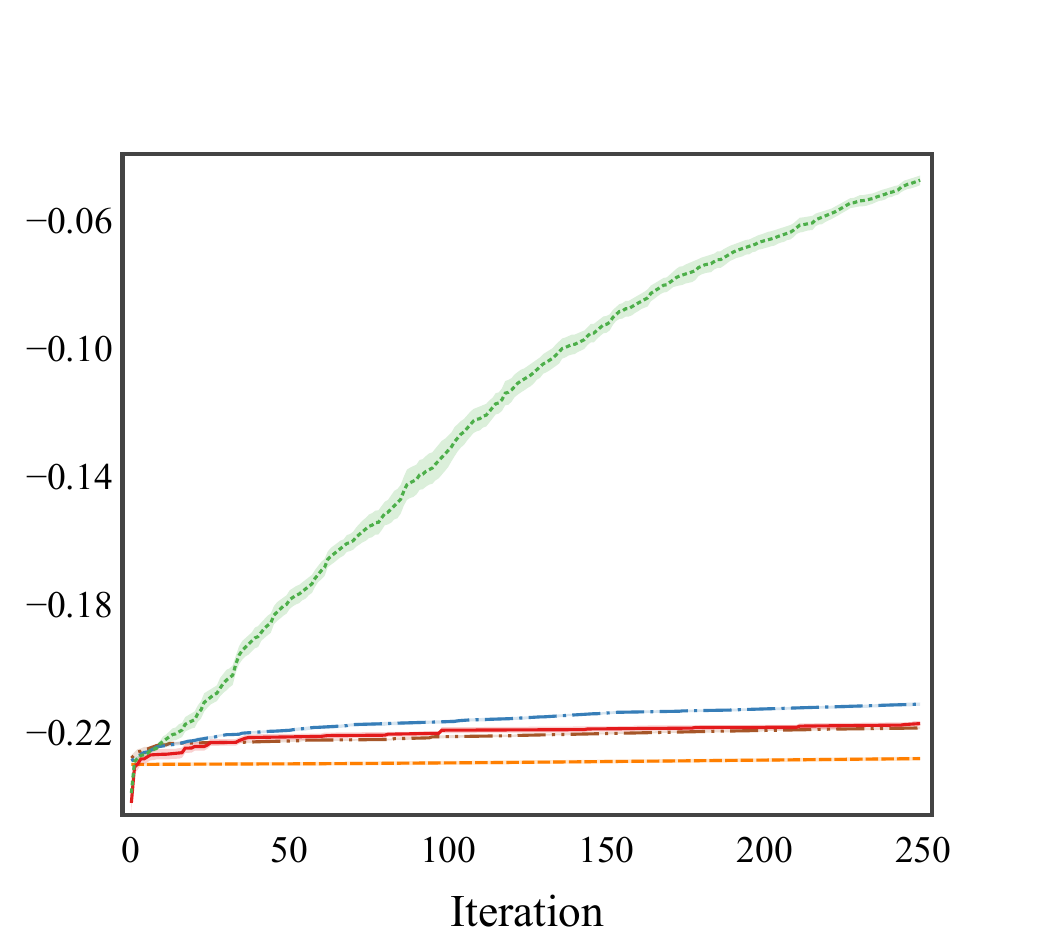}
        \label{fig:NYC-12x20-}
    }%
    \subfigure[grid $=15\times26$, $|A|=235$]{%
        \includegraphics[width=0.32\textwidth, trim={0 0 1.7cm 2.5cm}, clip]{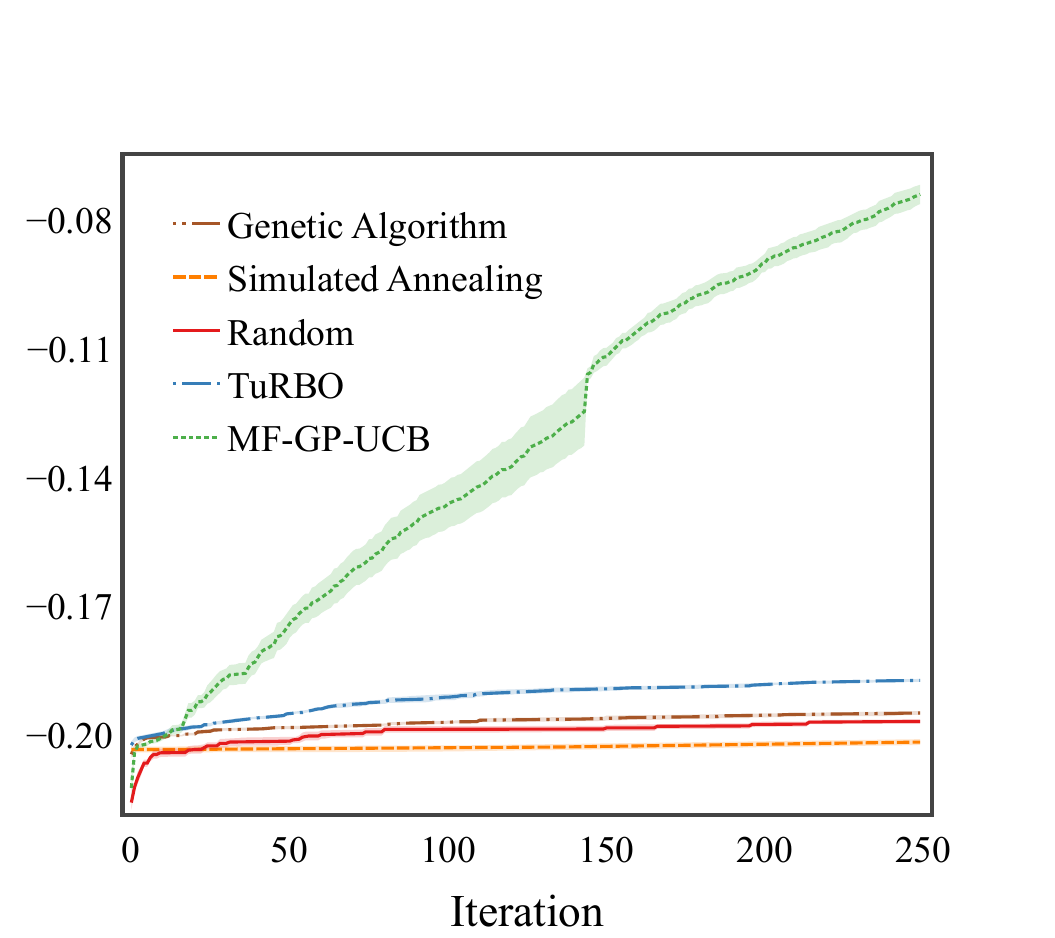}
        \label{fig:NYC-15x26-}
    }%
    \subfigure[grid $=20\times36$, $|A|=394$]{%
        \includegraphics[width=0.32\textwidth, trim={0 0 1.7cm 2.5cm}, clip]{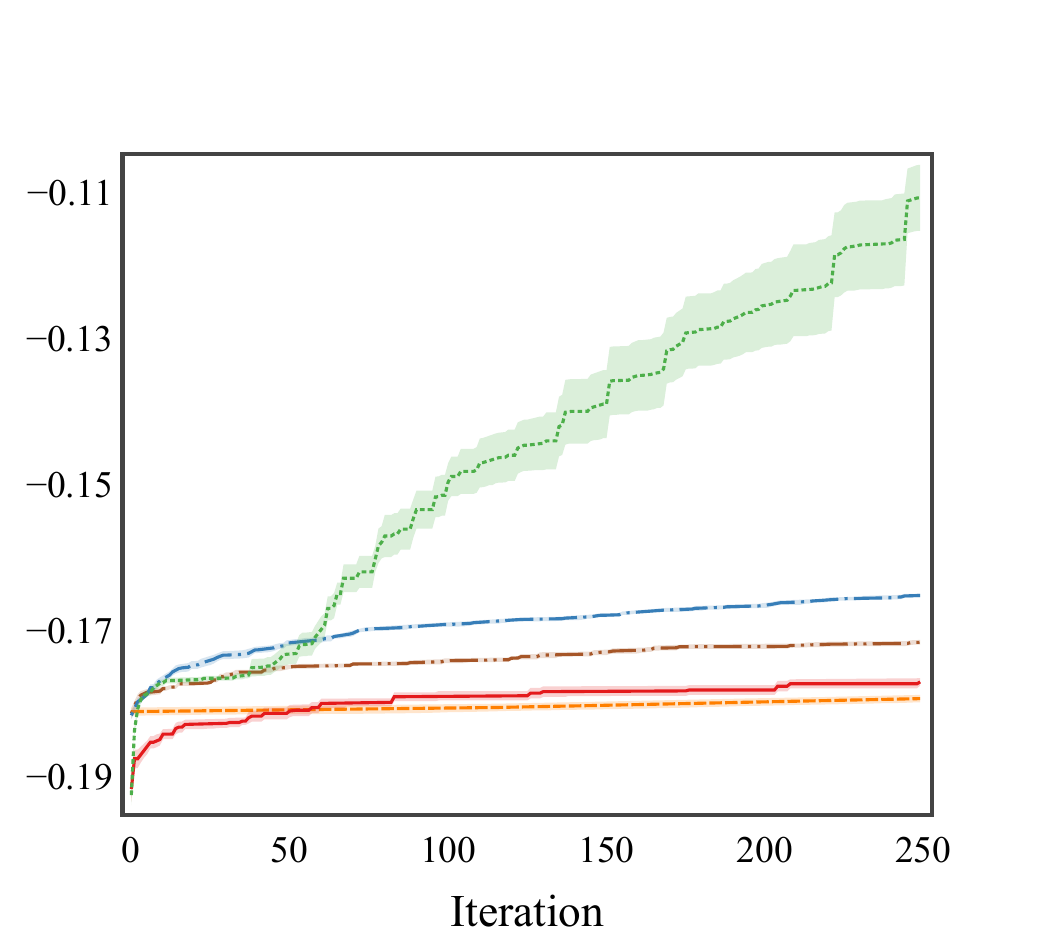}
        \label{fig:NYC-20x36-}
    }%
    \caption{NYC Taxi -- the scaling power of MF-GP-UCB is made more apparent in this experiment: where TuRBO struggles to do much better than random in the 250 iteration regime, our algorithm finds great solutions in all three grid sizes. The black-box dimension of this experiment is $M=20,000$ with the denoted action spaces and a single context $|C|=1$. This experiment is a scaled-up version of the LouVelo bike-sharing function.}
    \label{fig:NYC-Taxi}
\end{figure*}

We use this dataset \cite{nyc-data} to estimate the demand distribution for taxis across NYC. Pickup locations are provided in continuous space, so the service region was discretised into a mesh grid to accommodate a discrete action space. The idea is much the same as in \Cref{real-world:LouVelo} where each agent/taxi picks a point in the grid, and the joint payoff is computed as the negative of $D_{JS}$ between the empirical distribution of taxis, and the ground truth distribution of demand. Our service region is focused on the centre of Manhattan Island since this is where most of the available data is. The demand is estimated by analysis of the ride pickup locations and their corresponding timestamps. We specify, just as in the previous section, a specific weekday and a number of consecutive weekly samples. We chose twenty Fridays with an additionally specified rush-hour interval from 15:00 to 18:00 to create normalised pickup distributions. To estimate the normalised ground truth demand distribution, we take the average of these twenty samples. In this experiment, we scale the number of agents to a more impressive twenty thousand and use three variants of mesh-grid sizes \{$12\times20$, $15\times26$, $20\times36$\} corresponding to blocks of approximately $500\ \text{m}^2$, $400\ \text{m}^2$, and $300\ \text{m}^2$ respectively. The blocks with no demand are removed from the pool of possible actions, resulting in actual action space sizes of $\{160, 235, 394\}$. \Cref{fig:NYC-heatmaps} is a comparison between the ground truth demand distribution and solutions found by MF-GP-UCB, random search, and TuRBO in a $12\times20$ grid and $T=250$ iterations. Due to the nature of optimising over a $M=20,000$ dimensional space, the benchmarks struggle to improve much over random search, while MF-GP-UCB is in a league of its own due to the MF assumption. In \Cref{fig:NYC-Taxi}, we report the convergence of the NYC Taxi experiment. \looseness=-1

\subsubsection{Maritime Refuelling} \label{real-world:Maritime}
The maritime refuelling experiment leverages the use of agent contexts. Each agent/vessel samples a context (one of five regions in the world), and its action is a port to refuel in. Furthermore, each port in the dataset has a region attribute. We would like to distribute the vessels across the available ports in a way that minimises a) waiting times caused by congestion and b) costly endeavours of region switching. The agent incurs a penalty by choosing a port in a context/region that is different from its own. The system reward may also be expressed as: 
\[
    r(\mathbf{x})=\frac{1}{M}\sum_{m=1}^{M}\Bigl[ -0.5\overbrace{-\frac{P_{x_m}}{V_{x_m}}}^{\text{congestion}}\underbrace{+\mathds{1}_{\bm{c}_{m}}(\bm{c}_{x_m})}_\text{context bonus}\Bigr],
\]
where $P_{x_m}$ is the frequency of action $x_m$ across all agents (port occupancy), $V_{x_m}$ is the normalised capacity of port $x_m$ (estimated from dataset \citet{maritime-data}), and $\mathds{1}_{\bm{c}_{m}}$ is the indicator function that checks if the contexts are matching. Each context encompasses exactly a fifth of the total agent population, but the ports do not necessarily reflect this distribution. In other words, agents will sometimes be forced to balance between picking a crowded port (and reducing the reward for everyone picking that port), or changing their region and making a sacrifice for the greater good. They simultaneously have to appropriately disperse over the available ports while attempting to match contexts. This dynamic makes the problem interesting and challenging to optimise. Here, we use $M=3,000$ agents/vessels over an action space of $|A|=30$ randomly selected ports and $|C|=5$. The result is displayed in \Cref{fig:maritime}.

\section{Conclusion}
Within this paper, we introduced MF-GP-UCB, a novel BO algorithm designed to address the challenge of optimising large-scale cooperative agent systems with unknown, black-box payoff functions. MF-GP-UCB overcomes the limitations of traditional BO methods by leveraging the mean-field assumption, achieving a regret bound independent of the number of agents. This key advantage stems from exploiting the invariance of mean-field functions to permutations of agent actions. We provide a theoretical analysis establishing this improved regret bound. Furthermore, extensive empirical evaluations on diverse synthetic problems and real-world applications, including bike-sharing optimisation, taxi fleet distribution, and maritime vessel refuelling, demonstrate significant performance improvements and enhanced scalability compared to existing benchmark algorithms. In conclusion, MF-GP-UCB offers a powerful and scalable solution for mean-field, black-box optimisation in complex multi-agent systems. 

One limitation of our work is that we assumed a centralised controller and the ability to communicate with all of the agents so that the current policy can be shared.
A promising direction for future work is relaxing this assumption, allowing for a completely decentralised solution without the need for any communication during the runtime of the algorithm.
\looseness=-1
\section*{Impact Statement}
This paper presents work whose goal is to advance the field
of Machine Learning. There are many potential societal
consequences of our work, none which we feel must be
specifically highlighted here.

\section*{Acknowledgements}
Juliusz Ziomek was supported by the Oxford Ashton-Memorial Scholarship and EPSRC DTP grant EP/W524311/1. Matej Jusup acknowledges support from the Swiss National Science Foundation under the research project DADA/181210. Ilija Bogunovic was supported by the EPSRC New Investigator Award EP/X03917X/1; the Engineering and Physical Sciences Research Council EP/S021566/1; and Google Research Scholar award.

\bibliography{arXiv_paper}
\bibliographystyle{icml2025}

\onecolumn
\newpage
\appendix

\section{Proof of Theorem \ref{thm:ucb_mf_bo_bound_centralised}} \label{ap:ucb_mf_bo_bound_centralised_proof}
\begin{proof}
We will follow the general idea of \cite{takeno2023randomized} with some important modifications. Purely for the sake of analysis, at each timestep $t$, let us consider a discretisation $\Xi_t$ of the space of vectors representing possible distributions $\xi \in \mathbb{R}^{|A||C|}$ with each dimension split equally into $\tau_t = b|A||C|t^2(\log(a|A||C|) + \sqrt{\pi} / 2) $. Let $[\xi]_t$ denote the distribution in the discretisation $\Xi_t$ at time $t$ that is closest to $\xi$. 
Let us define $g(\xi) = \mathbb{E}_{\bm{x},\bm{c} \sim \xi(\bm{x}|\bm{c})p(\bm{c})}[f(\bm{x}, \bm{c}, \xi)]$. We can then decompose regret as:
\begin{align*}
    \mathbb{E}[R_T] &= \sum_{t=1}^T \Big(\mathbb{E}[g(\xi^\star) - g([\xi^\star]_t)] + \mathbb{E}[g([\xi^\star]_t)  -
    g(\xi_t)]\Big) \\
    & =  \underbrace{\sum_{t=1}^T \mathbb{E}[g(\xi^\star) - g([\xi^\star]_t)]}_{R_T^A} + \underbrace{\sum_{t=1}^T \mathbb{E}[g([\xi^\star]_t)  -
    g(\xi_t)]}_{R_T^B}   .
\end{align*}
We thus need to bound $R_T^A$ and $R_T^B$. Consider the second term for now. We have that:
\begin{align*}
    R_T^B &= \sum_{t=1}^T \mathbb{E}[g([\xi^\star]_t)  -
    g(\xi_t)] \\
    & = \sum_{t=1}^T \mathbb{E}_{\mathcal{D}_{t-1}}[\mathbb{E}[g([\xi^\star]_t)|\mathcal{D}_{t-1}]  - \alpha_{t}(\xi_t|\mathcal{D}_{t-1}) + \alpha_{t}(\xi_t|\mathcal{D}_{t-1}) - 
    \mathbb{E}[g(\xi_t)|\mathcal{D}_{t-1}]] \\
    & \le \sum_{t=1}^T \mathbb{E}_{\mathcal{D}_{t-1}}[\mathbb{E}[g([\xi^\star]_t)|\mathcal{D}_{t-1}]  - \alpha_{t}([\xi^\star]_t|\mathcal{D}_{t-1}) + \alpha_{t}(\xi_t|\mathcal{D}_{t-1}) - 
    \mathbb{E}[g(\xi_t)|\mathcal{D}_{t-1}]] \\
    & = \sum_{t=1}^T \mathbb{E}_{\mathcal{D}_{t-1}}[\mathbb{E}[g([\xi^\star]_t) - \alpha_{t}([\xi^\star]_t|\mathcal{D}_{t-1})|\mathcal{D}_{t-1}]   + \mathbb{E}[\alpha_{t}(\xi_t|\mathcal{D}_{t-1}) - 
    g(\xi_t)|\mathcal{D}_{t-1}]] \\
    & =  \underbrace{\sum_{t=1}^T\mathbb{E}[g([\xi^\star]_t) - \alpha_{t}([\xi^\star]_t|\mathcal{D}_{t-1})]}_{R_T^C}   + \underbrace{\sum_{t=1}^T\mathbb{E}[\alpha_{t}(\xi_t|\mathcal{D}_{t-1}) - 
    g(\xi_t)]}_{R_T^D} ,
\end{align*}
where the inequality follows from the fact that due to acquisition rule $\xi_t = \max_{\xi \in \Delta_A^C \alpha(\xi_t|\mathcal{D}_{t-1})}$, we must have that $\alpha(\xi|\mathcal{D}_{t-1}) \ge  \alpha([\xi^\star]_t|\mathcal{D}_{t-1})$.
Let us summarise the proof so far. We decomposed the regret as follows:
\begin{align*}
    \mathbb{E}[R_T] \le R_T^A + R_T^B \le R_T^A + R_T^C + R_T^D .
\end{align*}
We will now bound each the of terms above, within the respective subsections below. Combining the resulting bounds will give:
\begin{align*}
    \mathbb{E}[R_T] &\le R_T^A + R_T^C + R_T^D\\
    &\le \frac{\pi^2}{6} + |A|\sqrt{2|A||C|}\min\{\frac{1}{\mathcal{G}},|A||C|\} \max\left\{ 2\zeta^\prime(2), \frac{\pi^2}{6} \right\} +\beta_T \sqrt{TD\gamma_T} + \frac{\pi^2}{6}  \\
    &= \mathcal{O}(\beta_T \sqrt{T\gamma_T} + \mathcal{B}),
\end{align*}
where $\mathcal{B} = \sqrt{2|A|^3|C|}\min\{(b(\log(a|A||C|) + \sqrt{\pi} / 2))^{-1},|A||C|\}$, $D>0$ is a constant, $\mathcal{G} = b(\log(a|A||C|) + \sqrt{\pi} / 2)$ and $\zeta(\cdot)$ is the Riemann zeta function.
\end{proof}
\subsection{Bounding $R_T^A$}
We decompose the difference between expected $g(\xi^\star)$ and $g([\xi^\star]_t)$ into the part coming from the difference in arguments to function $f(\cdot)$ and the part coming from difference of sampling distributions for $\bm{x}$. We do it below:
\begin{align*}
R_T^A &= \sum_{t=1}^T \mathbb{E}[g(\xi^\star) - g([\xi^\star]_t)]\\
    &\le \sum_{t=1}^T \Big( \mathbb{E}[\mathbb{E}_{\bm{c} \sim p(\bm{c})}[\mathbb{E}_{\bm{x} \sim \xi^\star(\bm{x}|\bm{c})}[f(\bm{x}, \bm{c}, \xi^\star)] - \mathbb{E}_{\bm{x} \sim [\xi^\star]_t(\bm{x}|\bm{c})}[f(\bm{x}, \bm{c}, [\xi^\star]_t)]]\Big) \\
    &= \sum_{t=1}^T \Big( \mathbb{E}[\mathbb{E}_{\bm{c} \sim p(\bm{c})}[\mathbb{E}_{\bm{x} \sim \xi^\star(\bm{x}|\bm{c})}[f(\bm{x}, \bm{c}, \xi^\star) - f(\bm{x}, \bm{c}, [\xi^\star]_t)] + \sum_{\bm{x} \in A}f(\bm{x}, \bm{c}, [\xi^\star]_t])([\xi^\star]_t(\bm{x}|\bm{c}) - \xi^\star(\bm{x}|\bm{c})  )]\Big) \\
    &\le \sum_{t=1}^T \Big( \mathbb{E}[\mathbb{E}_{\bm{c} \sim p(\bm{c})}[\mathbb{E}_{\bm{x} \sim \xi^\star(\bm{x}|\bm{c})}[f(\bm{x}, \bm{c}, \xi^\star) - f(\bm{x}, \bm{c}, [\xi^\star]_t)] + \max_{\bm{x},\bm{c}^\prime \in A\times C}\lVert [\xi^\star]_t(\bm{x}|\bm{c}^\prime) - \xi^\star(\bm{x}|\bm{c}^\prime) \rVert_1 \sum_{\bm{x} \in A}f(\bm{x}, \bm{c}, [\xi^\star]_t) ]\Big) .
\end{align*}    
Due to construction of the discretisation $\Xi_t$, we also have that:
\begin{align*}
    \sup_{\xi_t \in \Delta_C^A} \lVert [\xi^\star]_t(\bm{x}|\bm{c}) - \xi^\star(\bm{x}|\bm{c}) \rVert_1 \le \frac{|A||C|} {\tau_t} =   \frac{1}{bt^2(\log(a|A||C|) + \sqrt{\pi} / 2) } 
\end{align*}
We start by bounding the difference coming from different arguments to $f(\cdot)$. This part closely resembles the proof of \cite{takeno2023randomized}. We  have that :
\begin{align*}
    \sum_{t=1}^T\mathbb{E}[\mathbb{E}_{\bm{x},\bm{c} \sim \xi^\star(\bm{x}|\bm{c})p(\bm{c})}[f(\bm{x}, \bm{c}, \xi^\star) - f(\bm{x}, \bm{c}, [\xi^\star]_t)] &\le \sum_{t=1}^T \mathbb{E}[\sup_{\bm{x},\bm{c},\xi \in A \times C \times \Delta_C^A}[f(\bm{x}, \bm{c}, \xi^\star) - f(\bm{x}, \bm{c}, [\xi^\star]_t)] \\
    & \le \sum_{t=1}^T \mathbb{E}[L_{\textrm{max}}\sup_{\bm{x},\bm{c},\xi \in A \times C \times ^A} \lVert [\xi^\star]_t(\bm{x}|\bm{c}) - \xi^\star(\bm{x}|\bm{c}) \rVert_1] \\
    & \le\sum_{t=1}^T \mathbb{E}[L_{\textrm{max}}] \frac{1}{bt^2(\log(a|A||C|) + \sqrt{\pi} / 2) },
    \end{align*}
    where in the second line we used Assumption \ref{as:lkernel} and defined $L_{\textrm{max}} = \max_{\bm{x},\bm{c} \in A \times C}\sup_{\bm{x}^\prime,\bm{c}^\prime,\xi \in A \times C \times \Delta_C^A} \Big| \frac{\partial f}{\partial \xi(\bm{x}^\prime, \bm{c}, \xi)} \Big|$. As proven in Lemma H.1 of \cite{takeno2023randomized}, we have that $\mathbb{E}[L_{\textrm{max}}] \le b(\log(a|A||C|) + \sqrt{\pi} / 2)$. Plugging this into inequality above gives:
    \begin{align*}
    \sum_{t=1}^T \mathbb{E}[L_{\textrm{max}}] \frac{1}{bt^2(\log(a|A||C|) + \sqrt{\pi} / 2) } & \le\sum_{t=1}^T \frac{1}{t^2 } \\ 
    &\le \frac{\pi^2}{6}.
\end{align*}
It thus remains to bound the difference due to sampling distributions. 
\begin{align*}
    &\sum_{t=1}^T\mathbb{E}[\mathbb{E}_{\bm{c} \sim p(\bm{c})}[\max_{\bm{x},\bm{c}^\prime \in A\times C}\lVert [\xi^\star]_t(\bm{x}|\bm{c}^\prime) - \xi^\star(\bm{x}|\bm{c}^\prime) \rVert_1  \sum_{\bm{x} \in A}f(\bm{x}, \bm{c}, [\xi^\star]_t])] \Big) \\
    &\le \sum_{t=1}^T\mathbb{E}[\mathbb{E}_{\bm{c} \sim p(\bm{c})}[\sup_{\bm{x},\bm{c}^\prime,\xi \in A \times C \times \Delta_C^A} \lVert [\xi^\star]_t(\bm{x}|\bm{c}^\prime) - \xi^\star(\bm{x}|\bm{c}^\prime) \rVert_1 \sum_{\bm{x} \in A}\max_{\xi \in  \Xi_t}f(\bm{x}, \bm{c}, \xi)]]\\
    & \le \sum_{t=1}^T\frac{|A||C|}{\tau_t }\sum_{\bm{x} \in A}\mathbb{E}_{\bm{c} \sim p(\bm{c})}[\mathbb{E}_f[\max_{\xi \in  \Xi_t}f(\bm{x}, \bm{c}, \xi)]] .
\end{align*}
For a fixed $\bm{x}$ and $\bm{c}$, the expression $\mathbb{E}_f[\max_{\xi \in  \Xi_t}f(\bm{x}, \bm{c}, \xi)]$ is the expected maximum of $|\Xi_t|$ Gaussian variables (which are \textbf{not} independent). The marginal distribution of these variables is $f(\cdot) \sim \mathcal{N}(0, k(\cdot, \cdot))$  and due to Assumption on the kernel $k(\cdot, \cdot) \le 1$, we get  that the marginal variance  $\textrm{Var}(f(\cdot)) \le 1$. As such, we can employ Lemma 2.2 from \cite{devroye2001combinatorial}, which states that expected maximum of $n$ variables, which are $\sigma$-subgaussian is bounded by $\sigma \sqrt{2 \log n}$. Notably, this holds even if variables are dependent, like in our case. Employing this Lemma with $\sigma = k(\cdot, \cdot) \le 1$ and $n = |\Xi_t|$, we get:

\begin{align*}
    \sum_{t=1}^T\frac{|A||C|}{\tau_t }\sum_{\bm{x} \in A}\mathbb{E}_{\bm{c} \sim p(\bm{c})}[\mathbb{E}_f[\max_{\xi \in  \Xi_t}f(\bm{x}, \bm{c}, \xi)]] & \le \sum_{t=1}^T\frac{|A||C|}{\tau_t }\sum_{\bm{x} \in A}\mathbb{E}_{\bm{c} \sim p(\bm{c})}[\sqrt{2\log |\Xi_t|}] \\
    & =  \sum_{t=1}^T\frac{|A|^2|C|}{\tau_t }\sqrt{2|A||C|\log \tau_t} .
\end{align*}
Since we split the each dimension of the domain into $\tau_t$ pieces, $\tau_t$ cannot be smaller than $1$. We thus have:
\begin{align*}
    \sum_{t=1}^T\frac{|A|^2|C|}{\tau_t }\sqrt{2|A||C|\log \tau_t}  &\le \sum_{t=1}^T\frac{|A|^2|C|\sqrt{2|A||C|\log \tau_t}}{\max\{|A||C|bt^2(\log(a|A||C|) + \sqrt{\pi} / 2),1\} } \\
    & \le \sum_{t=1}^T\frac{|A|\sqrt{2|A||C|\log \tau_t}}{t^2\max\{b(\log(a|A||C|) + \sqrt{\pi} / 2),\frac{1}{|A||C|}\} } \\
    & \le \sum_{t=1}^T\frac{|A|\sqrt{2|A||C|\left(2\log t + \log b(\log(a|A||C|) + \sqrt{\pi} / 2)\right)}}{t^2\max\{b(\log(a|A||C|) + \sqrt{\pi} / 2),\frac{1}{|A||C|}\} } \\
    & \le \sum_{t=1}^T|A|\sqrt{2|A||C|}\frac{\sqrt{2\log t} + \sqrt{\log b(\log(a|A||C|) + \sqrt{\pi} / 2)}}{t^2\max\{b(\log(a|A||C|) + \sqrt{\pi} / 2),\frac{1}{|A||C|}\} } .
\end{align*}
For brevity, we will adopt the notation $\mathcal{G} = b(\log(a|A||C|) + \sqrt{\pi} / 2)$. Continuing from above we have:
\begin{align*}
\sum_{t=1}^T|A|\sqrt{2|A||C|}\frac{\sqrt{2\log t} + \sqrt{\mathcal{G}}}{t^2\max\{\mathcal{G},\frac{1}{|A||C|}\} }    & \le |A|\sqrt{2|A||C|} \sum_{t=1}^T\left(2\frac{\log t}{t^2 } \min\{\frac{1}{\mathcal{G}},|A||C|\} + \frac{1}{t^2}\min\{\frac{1}{\sqrt{\mathcal{G}}}, |A||C|\sqrt{\mathcal{G}}\}  \right) \\
& \le |A|\sqrt{2|A||C|}\min\{\frac{1}{\mathcal{G}},|A||C|\} \left(2 \sum_{t=1}^T\frac{\log t}{t^2 } + \sqrt{\mathcal{G}} \sum_{t=1}^T\frac{1}{t^2} \right).
\end{align*}
We now use the well-known inequality $\sum_{t=1}^T\frac{1}{t^2} \le \frac{\pi^2}{6}$ and observe that:
\begin{align*}
    \sum_{t=1}^T\frac{\log t}{t^2 } = \frac{d}{ds}\sum_{t=1}^T\frac{1}{t^s}\bigg|_{s=2} = \frac{d}{ds}\zeta(s)\bigg|_{s=2} = \zeta^\prime(2),
\end{align*}
where $\zeta$ is the famous Riemann zeta function. Thus continuing from previous display:
\begin{align*}
    |A|\sqrt{2|A||C|}\min\{\frac{1}{\mathcal{G}},|A||C|\} \left(2 \sum_{t=1}^T\frac{\log t}{t^2 } + \sqrt{\mathcal{G}} \sum_{t=1}^T\frac{1}{t^2} \right) &\le |A|\sqrt{2|A||C|}\min\{\frac{1}{\mathcal{G}},|A||C|\} \max\left\{ 2\zeta^\prime(2), \frac{\pi^2}{6} \right\} \\
\end{align*}

\subsection{Bounding $R_T^C$}
We will write $U_t(\bm{x}, \bm{c}, \xi) = \mu_{t-1}(\bm{x}, \bm{c}, \xi) +\beta_t \sigma_{t-1}(\bm{x}, \bm{c}, \xi) $.
\begin{align*}
     \sum_{t=1}^T\mathbb{E}[g([\xi^\star]_t) - \alpha_{t}([\xi^\star]_t|\mathcal{D}_{t-1})]&= \sum_{t=1}^T\mathbb{E}[\mathbb{E}_{\bm{x},\bm{c} \sim [\xi^\star]_t(\bm{x}|\bm{c})p(\bm{c})}[f(\bm{x}, \bm{c}, [\xi^\star]_t) - U_t(\bm{x}, \bm{c}, [\xi^\star]_t)] \\
     &\le \sum_{t=1}^T\mathbb{E}[\mathbb{E}_{\bm{x},\bm{c} \sim [\xi^\star]_t(\bm{x}|\bm{c})p(\bm{c})}[\left(f(\bm{x}, \bm{c}, [\xi^\star]_t) - U_t(\bm{x}, \bm{c}, [\xi^\star]_t)\right)^+] \\
     &\le \sum_{t=1}^T\sum_{\xi \in \Xi_t}\sum_{\substack{\bm{x} \in A \\ \bm{c} \in C}}\mathbb{E}[ \big(f(\bm{x}, \bm{c}, \xi) - U_t(\bm{x}, \bm{c}, \xi)\big)^+]\\
     &= \sum_{t=1}^T\sum_{\xi \in \Xi_t}\sum_{\substack{\bm{x} \in A \\ \bm{c} \in C}}\mathbb{E}[\mathbb{E}[ \big(f(\bm{x}, \bm{c}, \xi) - U_t(\bm{x}, \bm{c}, \xi)\big)^+]|\mathcal{D}_{t-1}] .
    \end{align*}
    
    Note that conditionally on $\mathcal{D}_{t-1}$, we have that $f(\bm{x}, \bm{c}, \xi) \sim \mathcal{N}(\mu_{t-1}(\bm{x}, \bm{c}, \xi), \sigma_{t-1}^2(\bm{x}, \bm{c}, \xi))$ and thus $f(\bm{x}, \bm{c}, \xi) - U_t(\bm{x}, \bm{c}, \xi) \sim \mathcal{N}(-\beta_t\sigma_{t-1}^2(\bm{x}, \bm{c}, \xi), \sigma_{t-1}^2(\bm{x}, \bm{c}, \xi))$. For a Gaussian variable $Z \sim \mathcal{N}(m,s^2)$, we have that:
    \begin{equation*}
        \mathbb{E}[(Z)^+] \le \frac{s}{\sqrt{2\pi}} \exp\left(-\frac{m^2}{2s^2}\right).
    \end{equation*}
Using this fact for the variable $\big(f(\bm{x}, \bm{c}, \xi) - U_t(\bm{x}, \bm{c}, \xi)\big)^+$, yields the following:
    
    \begin{align*}
     \sum_{t=1}^T\sum_{\xi \in \Xi_t}\sum_{\substack{\bm{x} \in A \\ \bm{c} \in C}}\mathbb{E}[\mathbb{E}[ \big(f(\bm{x}, \bm{c}, \xi) - U_t(\bm{x}, \bm{c}, \xi)\big)^+]|\mathcal{D}_{t-1}] & \le \sum_{t=1}^T\sum_{\xi \in \Xi_t}\sum_{\substack{\bm{x} \in A \\ \bm{c} \in C}}\mathbb{E}\Big[\frac{\sigma_{t-1}(\bm{x}, \bm{c}, \xi)}{\sqrt{2\pi}} \exp \big(-\frac{\beta_t}{2} \big)\Big]  \\
     & =\sum_{t=1}^T|A||C||\Xi_t|\frac{1}{\sqrt{2\pi}} \exp\big(-\frac{\beta_t}{2}\big)\\
     & \le \sum_{t=1}^T \frac{1}{t^2} \\
     &\le \frac{\pi^2}{6}
\end{align*}

\subsection{Bounding $R_T^D$}
 We will again write $U_t(\bm{x}, \bm{c}, \xi) = \mu_{t-1}(\bm{x}, \bm{c}, \xi) +\beta_t \sigma_{t-1}(\bm{x}, \bm{c}, \xi) $. We note, that due to the fact that $f$ is a sampled from a GP, we have that conditioned on $\mathcal{D}_{t-1}
$, the function at any given point follows a distribution
$f(\cdot) \sim \mathcal{N}(\mu_{t-1}(\cdot), \sigma^2(\cdot))$. Additionally, $\xi_t$ becomes deterministic conditioned on $\mathcal{D}_{t-1}$. As such, we have that conditioned on $\mathcal{D}_{t-1}$, $\xi_t$ is independent of the function $f$. This allows us to write:
\begin{align*}
    \sum_{t=1}^T\mathbb{E}[\alpha_{t}(\xi_t|\mathcal{D}_{t-1}) - 
    g(\xi_t)] &= \sum_{t=1}^T\mathbb{E}[\mathbb{E}_{\bm{x},\bm{c} \sim \xi_t(\bm{x}|\bm{c})p(\bm{c})}[U_t(\bm{x}, \bm{c}, \xi_t) - f(\bm{x}, \bm{c}, \xi_t)]] \\
&= \sum_{t=1}^T\mathbb{E}_{\mathcal{D}_{t-1}}[\mathbb{E}_{f}[\mathbb{E}_{\bm{x},\bm{c} \sim \xi_t(\bm{x}|\bm{c})p(\bm{c})}[U_t(\bm{x}, \bm{c}, \xi_t) - f(\bm{x}, \bm{c}, \xi_t)]]|\mathcal{D}_{t-1}] \\
& =\sum_{t=1}^T\mathbb{E}_{\mathcal{D}_{t-1}}[\mathbb{E}_{\bm{x},\bm{c} \sim \xi_t(\bm{x}|\bm{c})p(\bm{c})}[U_t(\bm{x}, \bm{c}, \xi_t) - \mathbb{E}_{f}[f(\bm{x}, \bm{c}, \xi_t)|\mathcal{D}_{t-1}, \xi_t]|\mathcal{D}_{t-1}]] \\
& = \sum_{t=1}^T\mathbb{E}_{\mathcal{D}_{t-1}}[\mathbb{E}_{\bm{x},\bm{c} \sim \xi_t(\bm{x}|\bm{c})p(\bm{c})}[U_t(\bm{x}, \bm{c}, \xi_t) - \mu_{t-1}(\bm{x}, \bm{c}, \xi_t)|\mathcal{D}_{t-1}]] \\
& = \sum_{t=1}^T\mathbb{E}_{\mathcal{D}_{t-1}}[\mathbb{E}_{\bm{x},\bm{c} \sim \xi_t(\bm{x}|\bm{c})p(\bm{c})}[\beta_t\sigma_{t-1}(\bm{x}, \bm{c}, \xi_t)|\mathcal{D}_{t-1}]] .
\end{align*}
We now note that due to the fact that representative agent's action $x_t$ follows the distribution $\xi_t$ and their context $\bm{c}_t$ follows distribution $p(\bm{c})$, we can write:
\begin{align*}
\sum_{t=1}^T\mathbb{E}_{\mathcal{D}_{t-1}}[\mathbb{E}_{\bm{x},\bm{c} \sim \xi_t(\bm{x}|\bm{c})p(\bm{c})}[\beta_t\sigma_{t-1}(\bm{x}, \bm{c}, \xi_t)|\mathcal{D}_{t-1}]] & = \sum_{t=1}^T\mathbb{E}[\beta_t\sigma_{t-1}(\bm{x}_t, \bm{c}_t, \xi_t)] \\
&\le \beta_T \mathbb{E}[\sum_{t=1}^T\sigma_{t-1}(\bm{x}_t, \bm{c}_t, \xi_t)] \\
& \le \beta_T \mathbb{E}[\sqrt{T\sum_{t=1}^T\sigma^2_{t-1}(\bm{x}_t, \bm{c}_t, \xi_t)}] \\
& \le \beta_T \mathbb{E}[\sqrt{TD\gamma_T}] ,
\end{align*}
for some positive constant $D > 0$, where the penultimate inequality is due to Cauchy-Schwarz and the last inequality is due to Lemma 5.4 of \cite{srinivas2009gaussian}.

\end{document}